\documentclass[preprint,12pt,authoryear]{elsarticle}

\usepackage{amssymb}
\usepackage{amsmath}
\usepackage{amsthm}

\usepackage{algorithm}
\usepackage{algorithmic}
\usepackage[dvipsnames]{xcolor}
\usepackage{float}
\usepackage{bbm}

\usepackage{hyperref}
\hypersetup{colorlinks=true,
linkcolor=blue,
anchorcolor=blue,
citecolor=blue}

\newtheorem{theorem}{Theorem}
\newtheorem{proposition}{Proposition}
\newtheorem{assumption}{Assumption}

\newcommand{\cX}{\mathcal{X}}

\newcommand{\cI}{\mathcal{I}}

\newcommand{\bX}{\mathbf{X}}
\newcommand{\bx}{\mathbf{x}}
\newcommand{\bS}{\mathbf{S}}

\newcommand{\bi}{\mathbf{1}}

\newcommand{\E}{\mathbb{E}}
\newcommand{\R}{\mathbb{R}}

\renewcommand{\P}{\mathbb{P}}

\usepackage{lineno}

\begin{document}

\begin{frontmatter}



\title{Reliable Real-Time Value at Risk Estimation via Quantile Regression Forest with Conformal Calibration} 


\author[1,2]{Du-Yi Wang}
\author[1]{Guo Liang}
\author[3]{Kun Zhang}
\author[2]{Qianwen Zhu}

\affiliation[1]{organization={Institute of Statistics and Big Data, Renmin University of China},
            state={Beijing},
            country={China}}

\affiliation[2]{organization={Department of Decision Analytics and Operations, City University of Hong Kong},
            city={Kowloon},
            state={Hong Kong},
            country={China}}

\affiliation[3]{organization={School of Information, Renmin University of China},
            state={Beijing},
            country={China}}
            
\begin{abstract}
Rapidly evolving market conditions call for real-time risk monitoring, but its online estimation remains challenging. 
In this paper, we study the online estimation of one of the most widely used risk measures, Value at Risk (VaR).  Its accurate and reliable estimation is essential for timely risk control and informed decision-making.
We propose to use the quantile regression forest in the offline-simulation-online-estimation (OSOA) framework. Specifically, the quantile regression forest is trained offline to learn the relationship between the online VaR and risk factors, and real-time VaR estimates are then produced online by incorporating observed risk factors. To further ensure reliability, we develop a conformalized estimator that calibrates the online VaR estimates.
To the best of our knowledge, we are the first to leverage conformal calibration to estimate real-time VaR reliably based on the OSOA formulation.
Theoretical analysis establishes the consistency and coverage validity of the proposed estimators.
Numerical experiments confirm the proposed method and demonstrate its effectiveness in practice.
\end{abstract}



\end{frontmatter}



\section{Introduction}

Risk management is central to the daily operations of financial institutions, where accurate real-time risk measurement supports trading limits, margin requirements, capital allocation, and regulatory compliance \citep[see, e.g.,][]{jiang2020online,huang2022nonparametric,valle2025portfolio}. 
A famous example is J.P. Morgan's ``4:15 report,'' which requires firm-wide risk aggregation within 15 minutes after market close to manage overnight exposures and inform next-day trading decisions \citep{jorion2006,jiang2020online}. 
Notably, \citet{jiang2020online} refer to this setting as the online risk monitoring problem.

In this online monitoring context, the risk measure, a function of portfolio losses, is computed after observing the current risk factors, such as stock prices, interest rates, and exchange rates. 
Thus the computation can be viewed as estimating a conditional risk measure, i.e., a risk measure conditional on observed risk factors.
\citet{jiang2020online} adopt the offline-simulation-online-estimation (OSOA) framework \citep{hong2019offline}. 
{Specifically, offline, one simulates pairs of risk factors and portfolio losses to learn their functional relationship; then, online, one estimates risk by inputting real-time risk factors into the learned relationship.}
In fact, different conditional risk measures correspond to different functional relations. For example, \citet{jiang2020online} study the probability of large losses conditional on real-time risk factors, and propose to use logistic regression to learn the functional relation.
In this paper, we focus on Value at Risk (VaR), one of the most widely used risk measures, and its online estimation with theoretical justification.

VaR is defined as a quantile of the probability distribution of the portfolio loss at a given confidence level. Accordingly, the online VaR estimation is essentially the conditional quantile estimation, and, further, under the OSOA framework, it amounts to learn a conditional quantile function, that is, quantile regression.
Moreover, beyond accuracy and timeliness, the reliability of risk measures is particularly crucial, especially in the context of VaR \citep[see, e.g.,][]{BCBS1996Backtesting,Campbell2006}.
Reliability differs from accuracy because it distinguishes underestimation from overestimation for risk measures.
For example, over 200 trading days, a $99\%$ daily VaR should produce approximately two exceptions on average, where an exception is recorded when the trading outcome exceeds the estimate.
One or three exceptions may result in the same evaluation errors under symmetric accuracy metrics; however, in practice, one exception ({{VaR overestimation}}) is preferable to three ({{VaR underestimation}}) and is therefore considered more reliable. 
This preference arises from the fact that underestimation can trigger supervisory responses, including increases in the multiplication factor applicable to the internal models capital requirement, and in more severe cases, requirements to revise the model.
By contrast, VaR overestimation leads to less exceptions and is less likely to elicit supervisory action \citep{BCBS1996Backtesting}.

{In this paper, our goal is to estimate VaR in real time with theoretical guarantees.
The problem is challenging due to three aspects: {accuracy}, {timeliness}, and {reliability}.
First, to address accuracy, we adopt quantile regression forests (QRF) \citep{meinshausen2006quantile} to learn the conditional quantile function.
Note that large-scale portfolios induce high-dimensional risk factors \citep{broadie2015regression}. 
Moreover, since the mapping from risk factors to portfolio losses is highly nonlinear \citep{broadie2015regression}, the associated conditional quantile is inherently complex.
Thus, QRF is appropriate because it performs well in high-dimensional structures and can capture complex conditional quantile surfaces \citep{gnecco2024extremal}.
In addition, QRF has the advantages of requiring little tuning and possessing relatively well-understood statistical properties, compared to other machine learning methods \citep{GRF2019,grinsztajn2022tree}.
Second, timeliness is achieved by applying the OSOA framework, since prediction via the learned function is sufficient to meet the online time limit.
Third, to ensure reliability, we propose to incorporate conformal calibration into QRF, which is a technique that has been increasingly studied in the machine learning \citep{angelopoulos2023conformal} and large language models \citep{Cherian2024,Mohri2024} literature in recent years .
The calibration recovers the target coverage level by correcting the QRF estimator using a held-out calibration set that empirically characterizes the error of the original QRF estimator.

In our algorithm, we train the QRF model offline on simulated pairs of risk factors and portfolio losses, and then estimate VaR online by evaluating the fitted model at the observed risk factor.
Specifically, for QRF, in the training phase, it constructs a weighted empirical conditional distribution using the neighborhood observations in the leaf node across trees, and in the prediction phase, it estimates conditional quantiles from this distribution.
Furthermore, after training the QRF model, we compute residuals relative to the QRF estimate on a calibration set and adjust the original estimate by the corresponding empirical quantile of these residuals, yielding the conformalized estimate.
Our theoretical analysis establishes consistency for the proposed QRF estimator and provides coverage guarantees for the conformalized estimator.

We summarize our contributions as follows.
\begin{enumerate}
    \item 
    We propose a QRF estimator for the online VaR estimation problem. 
    Since online VaR corresponds to a conditional quantile, we formulate the problem as a quantile regression task within the OSOA framework and solve it using QRF.
    The OSOA framework ensures timeliness by moving time-consuming computation to the offline stage, so online estimation requires only fast evaluation of the learned regression model. 
    Moreover, QRF enhances accuracy, as it is well suited for complex and high-dimensional conditional quantile functions.

    \item We further develop a conformalized QRF estimator that offers improved reliability.
    Conformal calibration quantifies the exceedance error of QRF and compensates for this error to restore the coverage level.
    We compute calibration residuals relative to the QRF estimate on a held-out set, and correct the estimate by an empirical quantile of these residuals, so the resulting conformalized estimate attains the target level.
    To the best of our knowledge, despite that conformal prediction has been widely studied in machine learning \citep{angelopoulos2023conformal}, its potential in risk measurement remains underexplored. 
    In this paper, we highlight a paradigm that leverages conformal calibration to recover the target coverage and thereby achieve the reliability of risk measures, using the OSOA-based formulation as a bridge between quantile regression and online VaR estimation.

    \item We provide a theoretical analysis of the performance of two proposed estimators. 
    We establish pointwise and $L^2$ consistency for the QRF estimator, which indicates that as the offline sample size grows, the estimator asymptotically converges to the true conditional quantile both for each fixed risk factor and on average over risk factors.
    Moreover, we show coverage guarantees for the conformal QRF estimator. 
    The estimator yields finite-sample, model-free marginal coverage validity, that is, for any real-time risk factors, the corresponding portfolio loss falls below the estimate with at least the target probability, for any finite calibration sample size and any regression models.
    Further, this guarantee strengthens from marginal to conditional validity as the offline sample size increases.
    These theoretical results ensure the accuracy and reliability of our estimates, thereby making subsequent managerial decisions more informed and convincing.
    \item We conduct numerical experiments to verify the performance of our method.
    We use comprehensive metrics, including mean root integrated squared error, mean pinball loss, and mean coverage rate, to thoroughly evaluate the estimators on a realistic portfolio risk measurement example.
    Numerical results are provided to demonstrate the effectiveness of our method and to validate the theoretical analysis of the proposed estimators. 
\end{enumerate}
}

\subsection{Literature Review}
Our work is related to three lines of literature.

\textbf{Offline simulation online estimation.}
The OSOA framework is proposed by \citet{hong2019offline}, and has motivated a growing literature \citep{jiang2020online,cai2020online,jiang2024real,luo2024reliable,keslin2025ranking,zhang2025conditional,lin2025contextual}.
Among this body of work, most online applications are optimization tasks, whereas only a few address portfolio risk measure estimation problems.
Specifically, \citet{jiang2020online} use logistic regression to estimate the probability of large losses and to classify the risk category. They theoretically show that both the probability of a large deviation of the estimated probability from its true value and the probability of misclassification converge exponentially fast as the offline sample size goes to infinity.
Moreover, \citet{cai2020online} apply natural gradient boosting to learn the conditional distribution, and then estimate different types of risk measures. 
Nonetheless, they do not establish consistency for the estimators since gradient boosting methods are used \citep{duan2020ngboost,velthoen2023gradient}.
In this paper, we focus on online VaR estimation and establish theoretical guarantees.

\textbf{Quantile regression.}
Quantile regression is introduced by \citet{koenker1978regression}, while early parametric quantile regression may introduce additional bias when the conditional quantile surface is complex \citep{koenker2017quantile,gnecco2024extremal}. Nonparametric kernel \citep{yu1998local} and spline methods \citep{koenker1994quantile} are more flexible but degrade quickly in moderate to high dimensions due to the curse of dimensionality \citep{gnecco2024extremal}.
Furthermore, advances in machine learning have led to a surge in quantile regression methods, including gradient boosting \citep{friedman2001greedy,velthoen2023gradient}, neural network \citep{taylor2000quantile}, and random forest \citep{meinshausen2006quantile, GRF2019, gnecco2024extremal}. 
These methods perform well in situations with high-dimensional covariates and complex conditional quantile surfaces \citep{gnecco2024extremal}.

\textbf{Conformalized quantile regression.}
Conformal prediction is a general methodology for constructing prediction intervals with finite-sample coverage guarantees \citep[see, e.g.,][]{angelopoulos2023conformal}, and has been applied in large language models for uncertainty quantification \citep{Cherian2024,Mohri2024}.
Recent work by \citet{romano2019conformalized} introduces conformalized quantile regression (CQR), which produces the predition intervals that adapt to local variability in highly heteroscedastic settings.
Importantly, the conformal framework is agnostic to the underlying quantile regression methods and can therefore be naturally integrated with QRF.
Further theoretical analyses of CQR, including pointwise coverage properties and comparisons among different procedures, are provided in \citet{sesia2020comparison}. In this paper, we incorporate the conformal calibration technique into QRF to ensure the reliability of VaR estimators and establish the coverage guarantees.

The remainder of the paper is organized as follows. Section~\ref{sec:problem} formulates the real-time VaR estimation problem.
Section~\ref{sec:algorithm} presents the offline learning and online estimation algorithms based on QRF.
Section~\ref{sec:CGRF} develops the conformal QRF estimator.
Section~\ref{sec:theoretical} provides theoretical guarantees for the proposed estimators.
Section~\ref{sec:numerical} reports numerical results, and Section~\ref{sec:conclusion} concludes the paper.

\section{Problem Formulation}
\label{sec:problem}

Let $\bS(t) = (S_1(t),\ldots,S_d(t))^{\top}\in \R^d$ for $0\leq t\leq T$, denote the price dynamics of market variables at time $t$, such as interest rates, stock prices, exchange rates, and other tradable assets.
Let $\mathcal{F}_t$ be the filtration generated by the market evolution up to time $t$. 
We define three critical time points satisfying $0<u<\tau<T$ in the following.
\begin{itemize}
    \item $T$: the maturity of derivative instruments in the portfolio (e.g., options).
    \item $\tau$: the risk horizon at which the portfolio risk is to be evaluated.
    \item $u$: the real-time monitoring time when risk estimates are required.
\end{itemize}

Let $L(\tau)\in\mathbb{R}$ denote the portfolio loss at the risk horizon $\tau$, and let $\bX(t)= (X_1(t),\ldots,X_{\nu}(t))^{\top}\in\mathbb{R}^\nu$ for $\nu\geq d$ denote the risk factors at time $t$, $0\leq t\leq T$, so that $\bX(\tau)$ represents the risk factors at the risk horizon.
The portfolio typically contains derivative securities with nonlinear payoffs. Therefore, $L(\tau)$ is an expectation of (discounted) loss of a portfolio at maturity conditional on risk factors $\bX(\tau)$. By the Markovian property of $\bS(t)$, $\bX(\tau)$ can be constructed by $\bS(t)_{t\in[0,\tau]}$.
For example, in addition to $\bS(\tau)$, $\bX(\tau)$ may also include path-dependent state variables, such as the running maximum of the underlying price for barrier options or the running average for Asian options, as in \cite[see, e.g.,][]{lai2024simulating}.

At monitoring time $u$, we observe a realization of risk factors, denoted by $\bx(u)= (x_1(u),\ldots,x_{\nu}(u))^{\top}\in\mathbb{R}^\nu$, based on the available information $\mathcal{F}_u$. Then the online risk measure problem is to estimate the conditional risk measure of $L(\tau)$ given $\bX(u)=\bx(u)$:
\begin{align*}
    f_u(\bx(u)) = \rho\left(L(\tau) \mid \bX(u)=\bx(u)\right),
\end{align*}
where $\rho(\cdot)$ denotes a generic risk measure functional. In this paper, we are interested in online VaR estimation, where the target function is the conditional $\alpha$-quantile of $L(\tau)$ given $\bX(u) = \bx(u)$ with confidence level $\alpha\in(0,1)$:
\begin{align*}
    v_\alpha(\bx(u)) = f_u(\bx(u))
    & = Q_\alpha(L(\tau) \mid \bX(u) = \bx(u)) \nonumber\\ 
    & = \inf \left\{y \in \R: \P\left(L(\tau) \leq y \mid \bX(u) = \bx(u)\right) \geq \alpha\right\},
\end{align*}
where $Q_\alpha(\cdot \mid \cdot)$ denotes the conditional $\alpha$-quantile operator. 

Here we provide intuition for this target function to explain why the problem is nontrivial.
The portfolio loss function $L(\tau)$ is nonlinear and has no closed forms \citep{broadie2015regression,hong2017kernel,jiang2020online,wang2024smooth}.
Moreover, our target function is a conditional quantile of $L(\tau)$, which is inherently more difficult to estimate, since conditioning reduces effective sample size locally and tail quantiles further amplify data scarcity.

\section{The QRF Estimator}
\label{sec:algorithm}

To address real-time estimation challenges, we adopt the OSOA paradigm. The paradigm consists of two stages: an offline stage for heavy computation and an online stage for fast evaluation \citep{hong2019offline}. In the following, we provide a detailed description of each stage.

\begin{enumerate}
    \item \textbf{The Offline Stage: Simulation and Learning.}\\
    We simulate $n$ independent sample paths $\bS_i(t)_{t\in[0,\tau]}$, $i=1,\ldots,n$, from the underlying stochastic model, and meanwhile obtain the corresponding portfolio losses $L_i(\tau)$ by nested simulation.
    Specifically, given each generated path $\bS_i(t)_{t\in[0,\tau]}$, we sample $m$ independent inner paths $\bS_{i,j}(t)_{t\in[\tau,T]}$, $j=1,\ldots,m$, calculate the terminal payoff at time $T$ and the corresponding loss $l_{i,j}(\tau)$, and then compute the portfolio loss $L_i(\tau)$ using the average $\sum_{j=1}^m l_{i, j}(\tau)/m$.
    This procedure yields a dataset $\{(\bX_i(u),L_i(\tau))\}_{i=1}^n$. 
    \\
    Based on this data set, we train a conditional quantile regression model to learn the mapping from $\bX(u)$ to the conditional $\alpha$-quantile of $L(\tau)$, denoted by $\hat{f}_u(\cdot)$.

    \item \textbf{The Online Stage: Estimation.}\\
    After observing the real-world risk factor $\bx(u)$, we input it into the learned model, and return $\hat{f}_u(\bx(u))$ as the real-time VaR estimate.
\end{enumerate}

We make several remarks on this algorithm.
First, details of the nested simulation procedure for computing $L_i(\tau)$ can be found in \citet{gordy2010nested} and \citet{zhang2022bootstrap}, while the construction of $\bX_i(u)$ could follow \citet{lai2024simulating}.
Second, we do not account for the error from using a finite number of inner samples. This is reasonable because computational time is abundant during the offline stage, allowing us to allocate a sufficient budget to the inner simulations, which makes the inner error nearly negligible \citep{jiang2020online,cai2020online}.
Third, to learn the model ${f}_u(\cdot)$, we adopt QRF \citep{meinshausen2006quantile} as a quantile regression alternative within the OSOA framework. 
In the following, we present the detailed procedure of QRF.

In the offline stage, to obtain $\hat{f}_u(\cdot)$, we construct an ensemble of $B$ regression trees ${T(\theta_b)}_{b=1}^B$, where $\theta_b$ collects all randomness in the construction of tree $b$, using bootstrap resampling and randomized splitting rules.
In the online stage, the new data point $\bx(u)$ is propagated down each tree $T(\theta_b)$ to a leaf node $l(\bx(u),\theta_b)$ with the associated region $R_{l(\bx(u),\theta_b)}$. For each training sample $\bX_i(u)$, $i=1,\ldots,n$, its tree weight is defined as
\begin{align*}
    w_i(\bx(u),\theta_b) = \frac{\bi\{\bX_i(u)\in R_{l(\bx(u),\theta_b)}\}}{\#\{j:\bX_j(u)\in R_{l(\bx(u),\theta_b)}\}},
\end{align*}
where $\#\{\cdot\}$ denotes the cardinality of a set and $\mathbf{1}\{\cdot\}$ denotes the indicator function. Averaging across all trees then yields the forest weight
\begin{align*}
    w_i(\bx(u))=\frac{1}{B}\sum_{t=1}^B w_i(\bx(u),\theta_b).
\end{align*}
Therefore, the conditional $\alpha$-quantile, i.e., the VaR at confidence level $\alpha$ , is estimated as
\begin{align*}
    \hat{v}_\alpha(\bx(u))
    = \hat{f}_{u}(\bx(u))
    = \inf\left\{y\in\mathbb{R}:
    \sum_{i=1}^n w_i(\bx(u))\mathbf{1}\{L_i(\tau)\le y\}\ge\alpha
    \right\}.
\end{align*}
We refer to this estimator as the QRF estimator.

\section{The Conformal QRF Estimator}
\label{sec:CGRF}

The QRF estimator is designed to achieve accurate prediction, but does not distinguish between underestimation and overestimation.
However, in practice, overestimation is often preferred to underestimation.
For example, when estimated VaR is used to determine margin requirements, holding a larger margin buffer is generally preferable; otherwise, the institution may face substantial downside risk due to an insufficient capital cushion.
In a word, slight overestimation is often acceptable, whereas underestimation is undesirable.

Motivated by this consideration, we further develop an estimator to meet the coverage criteria. 
The main idea comes from CQR \citep{romano2019conformalized}, which provides a calibration step that adjusts a base quantile estimator to achieve a specified coverage level.
Our algorithm implements this idea as follows.
In the offline stage, given the dataset $\{(\bX_i(u), L_i(\tau))\}_{i=1}^n$, we split the index set into a training set and a calibration set disjointly, denoted by $\cI_1$ and $\cI_2$, respectively. 
The quantile regression model (e.g., QRF) is trained using only the data indexed by $\cI_1$ to obtain $\hat{v}_{\alpha}(\cdot)$ or say $\hat{f}_u(\cdot)$.
We then calculate the conformity scores on the calibration set $\cI_2$,
\begin{align*}
    E_i = L_i(\tau) - \hat{v}_{\alpha}(\bX_i(u)), \quad i\in \cI_2,
\end{align*}
and define $q_{\alpha}(E,\cI_2)$ as the empirical $\alpha$-quantile of the conformity scores $\{E_i\}_{i\in \cI_2}$. 
In the online stage, given a new observation $\bx(u)$, the conformal QRF estimator is given by
\begin{align*}
\hat{v}^{\mathrm{c}}_{\alpha}(\bx(u))=\hat{v}_{\alpha}(\bx(u))+q_{\alpha}(E,\cI_2).
\end{align*}

We offer several remarks on this estimator.
This procedure can be considered as a flexible calibration step applied after the conditional quantile regression, which conformalizes the initial regression estimate using the calibration set to improve coverage validity.
Moreover, due to its model-free nature, the calibration idea can be combined with a broad class of quantile regression methods beyond QRF.

\section{Theoretical Guarantees}
\label{sec:theoretical}

For notational convenience, we denote $\bX = \bX(u)$, $\bx = \bx(u)$, and $Y = L(\tau)$, ignoring the time index, which is not involved in our analysis.
We write $F(y \mid \bx)=\mathbb{P}(Y \leq y \mid \bX=\bx)$.
Moreover, $g(n)=o(f(n))$ stands for $g(n)/f(n)\to 0$, as $n\to\infty$. 
Furthermore, for two sequences of random variables $\left\{U_n\right\}$ and $\left\{V_n\right\}$, $U_n=o_p\left(V_n\right)$ denotes that $U_n / V_n \xrightarrow{p} 0$, as $n\to\infty$.
To facilitate analysis for the estimators, we make the following assumptions.

\begin{assumption}
\label{ass1}
    The support of $\bX$, denoted by $\cX\subset \R^{\nu}$, is bounded, and the density of $\bX$ is positive and bounded from above and below by positive constants.
\end{assumption}
Assumption \ref{ass1} is made about the distribution of risk factors, and is used in the literature such as \citet{wang2024smooth}. 

\begin{assumption}
\label{ass2}
Denote the size of a leaf node $l$ from a tree constructed with parameter $\theta$ by $k_\theta(l)=\#\left\{i \in\{1, \ldots, n\}: \bX_i \in R_{l(\bx, \theta)}\right\}$.
As $n\to \infty$, $\max _{l, \theta} k_\theta(l)=o(n)$ and $1 / \min _{l, \theta} k_\theta(l)=o(1)$.
\end{assumption}
Assumption~\ref{ass2} imposes two conditions on the leaf nodes. 
Specifically, the proportion of observations in any leaf node relative to all observations vanishes as $n$ grows, and the minimum number of observations in each leaf node diverges to infinity.

\begin{assumption}
\label{ass3}
Let $r$ be an arbitrary node in the tree, with $n_r$ observations. 
If $r$ is split, then there exists a constant $C_0>0$ such that, for any coordinate $m\in\{1,\ldots,\nu\}$, the probability that dimension $m$ is selected as the split variable is at least $C_0$.
Moreover, there exists a constant $\gamma\in(0,1/2]$ such that, whenever node $r$ is split, each resulting child node $r_c$ satisfies $n_{r_c}/ r_v \ge \gamma$.
\end{assumption}

Assumption~\ref{ass3} imposes two standard regularity conditions for leaf nodes in tree-based methods. 
The first condition ensures the probability of selecting any feature for splitting is not too small. 
The second condition ensures that the resulting child nodes do not have highly imbalanced sample sizes after the split.
Assumptions~\ref{ass2} and \ref{ass3} are standard regularity conditions in the literature related to random forests \citep{meinshausen2006quantile,wager2018estimation, GRF2019, CevidMichelNafBuhlmannMeinshausen2022}.

\begin{assumption}
\label{ass4}
For any $\bx, \bx^{\prime}\in\cX$, there exists a constant $C>0$ such that
\begin{align*}
\sup_{y\in\R}\left|F(y \mid \bx)-F\left(y \mid \bx^{\prime}\right)\right| \leq C\left\| \bx-\bx^{\prime} \right\|_1,
\end{align*}
where $\|\cdot \|_1$ denotes the $\ell_1$ norm.
\end{assumption}
Assumption~\ref{ass4} states that the conditional distribution function is Lipschitz continuous with respect to $\bx$ with parameter $C$.

\begin{assumption}
\label{ass5}
For any $\bx\in\cX$, $F(y \mid \bx)$ is strictly monotonically increasing in $y$.
\end{assumption}
Assumptions~\ref{ass4} and \ref{ass5} are used in \citet{meinshausen2006quantile}.

\begin{theorem}
\label{thm:consist}
Under Assumptions 1--5, for any $\bx\in \cX$, the proposed online estimator for VaR with confidence level $\alpha$ satisfies
\begin{align*}
    \hat{v}_\alpha(\bx) \xrightarrow{p} v_\alpha(\bx), \quad n \rightarrow \infty,
\end{align*}
where $n$ is the offline sample size.
\end{theorem}
{
\begin{proof}
In the following, we prove the result by applying the continuous mapping theorem for functionals \citep{wainwright2019high}. To this end, we need to verify two conditions.
With Assumptions 1-5, by Theorem 1 in \citet{meinshausen2006quantile}, we have 
\begin{align*}
\| \hat{F}_n(\cdot \mid \bx) - F(\cdot \mid \bx) \|_{\infty} = \sup_{y\in\R} |\hat{F}_n(y \mid \bx)-F(y \mid \bx) | \xrightarrow{p} 0,\quad n\to\infty,
\end{align*}
where $\hat{F}_n(y \mid \bx)=\sum_{i=1}^n w_i(\bx) \bi \left\{Y_i \leq y\right\}$. Thus, the first condition is checked.
Note that the conditional quantile estimator and true conditional quantile can be rewritten as a quantile functional $Q^{\alpha}$ of $\hat{F}_n(y \mid \bx)$ and ${F}(y \mid \bx)$, respectively, where $Q^\alpha(G)=\inf \{z \in \mathbb{R}: G(z) \geq \alpha\}$.
For convenience, we write $q=v_\alpha(\bx)$ and thus $F(q \mid \bx)=\alpha$.
From Assumption 5, we have that, for any $\epsilon>0$, $F(q-\epsilon \mid \bx) < F(q \mid \bx) < F(q+\epsilon \mid \bx)$.
Let $$\delta= \min\{\alpha-F(q-\epsilon\mid\bx), F(q+\epsilon\mid\bx)-\alpha\}>0.$$ 
If $\| \hat{F}_n(\cdot\mid\bx)-F(\cdot\mid\bx) \|_{\infty} \leq \delta$, then $\sup_{y\in \R} | \hat{F}_n(y\mid \bx) - F(y\mid\bx) | \leq \delta$, and thus for any $z\in\R$, $|\hat{F}_n(z\mid\bx) - F(z\mid\bx)|\leq \sup_{y\in \R}| \hat{F}_n(y\mid\bx) - F(y\mid\bx) |\leq \delta$.
In other words, $-\delta \leq \hat{F}_n(z|\bx) - F(z|\bx) \leq \delta$, that is, $F(z\mid\bx)-\delta \leq \hat{F}_n(z\mid\bx) \leq F(z\mid\bx)+\delta$. 
If $z\leq q-\epsilon$, we have $\hat{F}_n(z\mid\bx)\leq \hat{F}_n(q-\epsilon\mid\bx)\leq F(q-\epsilon\mid\bx) +\delta <\alpha$, where the last inequality comes from the definition of $\delta$. 
Similarly, if $z\geq q+\epsilon$, we have $\hat{F}_n(z\mid\bx)\geq \hat{F}_n(q+\epsilon\mid\bx)\geq F(q+\epsilon\mid\bx) -\delta > \alpha$.
Hence, $q-\epsilon<Q^{\alpha}(\hat{F}_n(y\mid\bx))<q+\epsilon$, that is, $|Q^{\alpha}(\hat{F}_n(y\mid\bx))- q |=|Q^{\alpha}(\hat{F}_n(y\mid\bx))- Q^{\alpha}(F(y\mid\bx)) |<\epsilon$. 
The second condition is verified that $Q^{\alpha}$ is continuous at $F$ with respect to the supremum norm using the definition.
Therefore, the proof is complete.
\end{proof}
}

Theorem \ref{thm:consist} establishes the pointwise consistency of the QRF estimator, ensuring that the estimated conditional quantile converges in probability to the true VaR as the offline sample size increases. While consistency ensures that the estimator is asymptotically correct, it allows for the possibility of rare but arbitrarily large estimation errors and provides no control over the moment of the estimation error, which limits its usefulness for online analysis. To strengthen this result, we impose an additional moment condition and establish convergence in the mean squared sense in Theorem \ref{theorem:L2}.

\begin{assumption}
    \label{ass:6}
    There exist constants $\delta>0$ and $M<\infty$, such that, for all $\bx$,
    \begin{align*}
    \E\!\left[|Y|^{2+\delta}\mid \bX=\bx\right]\le M.
    \end{align*}
\end{assumption}

\begin{theorem}
\label{theorem:L2}
Suppose that Assumptions~1--6 hold. 
Assume further that the tree construction satisfies an \emph{honesty} condition \citep[see, e.g.,][]{athey2016recursive,wager2018estimation}. 
Specifically, the randomization used to determine the tree structure, i.e., the sequence of splits, is independent of the response samples $\{Y_i\}_{i=1}^n$, so that the responses are used only to estimate quantities within the leaves of a fixed, data-independent partition.
Then, we have
\begin{align*}
    \lim_{n \to \infty} 
    \mathbb{E}\!\left[\big(\hat v_\alpha(\bX)-v_\alpha(\bX)\big)^2\right] = 0,
\end{align*}
where $n$ denotes the offline sample size.
\end{theorem}

\begin{proof}
The proof proceeds in three steps. We first establish convergence in probability at a random covariate $\bX$. 
We then derive uniform $(2+\delta)$-moment bounds for both $v_\alpha(\bX)$ and $\hat v_\alpha(\bX)$.
Finally, we combine these results to conclude $L^2$ convergence via uniform integrability.

\medskip
\noindent\textbf{Step 1: Convergence in probability at random $\bX$.}

Theorem \ref{thm:consist} shows that for any fixed $\bx$,
\begin{align*}
\hat v_\alpha(\bx)\xrightarrow{p} v_\alpha(\bx), \qquad n\to\infty.
\end{align*}
Let $\varepsilon>0$. Note that
\begin{align}
\label{eq:randX_prob_1}
\P\big(|\hat v_\alpha(\bX)-v_\alpha(\bX)|>\varepsilon\big) = \E\Big[\P\big(|\hat v_\alpha(\bX)-v_\alpha(\bX)|>\varepsilon\mid \bX\big)\Big].
\end{align}
For almost every realization of $\bX$, the conditional probability on the right-hand side converges to zero by pointwise consistency.
Since it is bounded between $0$ and $1$, the dominated convergence theorem yields
\begin{align}
\label{eq:conv_prob_randomX_rewrite}
\hat v_\alpha(\bX)-v_\alpha(\bX)\xrightarrow{p}0,\qquad n\to\infty.
\end{align}

\medskip
\noindent\textbf{Step 2: $(2+\delta)$-moment bounds.}

We first derive a moment bound for $v_\alpha(\bX)$. 
Let $p=2+\delta$ and $\beta=\min\{\alpha,1-\alpha\}$. 
By the definition of $v_\alpha(\bx)$,
\begin{align}
\label{eq:tail_lower_v_rewrite}
\P\big(|Y|\ge |v_\alpha(\bx)|\mid \bX=\bx\big)\ge \beta .
\end{align}
On the other hand, Markov's inequality implies
\begin{align*}
\P\big(|Y|\ge |v_\alpha(\bx)|\mid \bX=\bx\big)
\le
\frac{\E\big[|Y|^{p}\mid \bX=\bx\big]}{|v_\alpha(\bx)|^{p}}.
\end{align*}
Combining this inequality with \eqref{eq:tail_lower_v_rewrite} yields
\begin{align*}
|v_\alpha(\bx)|^{p}
\le \frac{1}{\beta}\E\big[|Y|^{p}\mid \bX=\bx\big]
\le \frac{M}{\beta},
\end{align*}
where the last inequality follows from Assumption~\ref{ass:6}. 
Consequently,
\begin{align}
\label{eq:v_moment_rewrite}
\E\big[|v_\alpha(\bX)|^{2+\delta}\big]<\infty.
\end{align}

We next derive a deterministic bound for $\hat v_\alpha(\bx)$.
For any $\bx$, we claim that for every realization of $\{Y_i\}_{i=1}^n$ and $\{w_i(\bx)\}_{i=1}^n$,
\begin{align}
\label{eq:hatv_ineq_rewrite}
|\hat v_\alpha(\bx)|^{p}
\le
\frac{1}{\beta}\sum_{i=1}^n w_i(\bx)\,|Y_i|^{p}.
\end{align}
To see this, note that the definition of $\hat v_\alpha(\bx)$ implies
\begin{align}
\label{eq:beta_mass_rewrite}
\sum_{i=1}^n w_i(\bx)\mathbf 1\big\{|Y_i|\ge |\hat v_\alpha(\bx)|\big\}\ge \beta,
\end{align}
which follows by considering the cases $\hat v_\alpha(\bx)\le 0$ and $\hat v_\alpha(\bx)\ge 0$ separately.
Applying the inequality
\begin{align*}
\mathbf 1\{|a|\ge t\}\le |a|^{p}/t^{p}, \qquad t>0,
\end{align*}
with $t=|\hat v_\alpha(\bx)|$ and combining with \eqref{eq:beta_mass_rewrite} yields \eqref{eq:hatv_ineq_rewrite}.

Taking expectations in \eqref{eq:hatv_ineq_rewrite}, conditional on $\bX$, $\{\bX_i\}_{i=1}^n$, and the tree structure $\theta_b$, honesty implies
\begin{align*}
\E\!\left[\sum_{i=1}^n w_i(\bX)|Y_i|^{p}\ \Big|\ \bX,\{\bX_i\}_{i=1}^n,\theta_b\right]
= \sum_{i=1}^n w_i(\bX)\E\big[|Y_i|^{p}\mid \bX_i\big] \le M.
\end{align*}
Hence,
\begin{align}
\label{eq:hatv_moment_rewrite}
\sup_{n\ge 1}\E\big[|\hat v_\alpha(\bX)|^{2+\delta}\big]\le \frac{M}{\beta}<\infty.
\end{align}

\medskip
\noindent\textbf{Step 3: Uniform integrability and $L^2$ convergence.}

Using the inequality $|a-b|^{p}\le 2^{p-1}(|a|^{p}+|b|^{p})$ with $p=2+\delta$, we obtain
\begin{align*}
\E\big[|\hat v_\alpha(\bX)-v_\alpha(\bX)|^{2+\delta}\big] \le 2^{1+\delta}\Big(\E|\hat v_\alpha(\bX)|^{2+\delta} + \E|v_\alpha(\bX)|^{2+\delta}\Big).
\end{align*}
Combining \eqref{eq:v_moment_rewrite} and \eqref{eq:hatv_moment_rewrite}, we obtain
\begin{align}
\label{eq:ui_bound_rewrite}
\sup_{n\ge 1}\E\big[|\hat v_\alpha(\bX)-v_\alpha(\bX)|^{2+\delta}\big]<\infty.
\end{align}
By Theorem 4.6.2 of \citet{Durrett_2019}, this bound implies that the family
$\{(\hat v_\alpha(\bX)-v_\alpha(\bX))^2\}_{n\ge 1}$ is uniformly integrable.

Together with the convergence in probability established in \eqref{eq:conv_prob_randomX_rewrite},
Theorem 4.6.3 of \citet{Durrett_2019} yields $L^1$ convergence of the squared error.
Consequently,
\begin{align*}
\E\Big[(\hat v_\alpha(\bX)-v_\alpha(\bX))^2\Big]\to 0,\qquad n\to\infty.
\end{align*}
This establishes the desired $L^2$ convergence.
\end{proof}

The additional moment condition in Assumption \ref{ass:6} excludes excessively heavy-tailed response distributions, ensuring that the conditional VaR remains well-behaved and that the quantile estimation error admits a finite second moment. Together with the honesty condition on tree construction, these assumptions yield convergence of the QRF-based VaR estimator in the mean squared sense. More importantly, this stronger form of convergence provides a fundamental building block for establishing the asymptotic validity and now we analyze the conformal QRF estimator $\hat{v}_{\alpha}^{\mathrm{c}}$.

\begin{proposition}
\label{proposition:finit_cvg}
If the samples $\{ (\bX_i, Y_i) \}_{i=1}^{n+1}$ are exchangeable, then
    \begin{align*}
        \P \{ Y_{n+1} \leq \hat{v}^{\mathrm{c}}_{\alpha}(\bX_{n+1}) \} \geq \alpha.
    \end{align*}
\end{proposition}

Proposition~\ref{proposition:finit_cvg} follows from Theorem 1 in \citet{romano2019conformalized}.
It provides a finite-sample, model-free marginal coverage probability guarantee for the conformalized online VaR estimator, if we regard $\bX_{n+1}$ as the observed risk factor online. 
In addition, note that exchangeability is weaker than the independent and identically distributed condition in our setting, and therefore, it is easy to satisfy.
Moreover, one may be interested in a stronger, conditional form, which Theorem~\ref{thm:asymp_conditional_conformal_var} formalizes by establishing an asymptotic conditional coverage guarantee.

\begin{theorem}
\label{thm:asymp_conditional_conformal_var}
Under the condition of Theorem \ref{theorem:L2}, assume additionally that the density of the conformity scores $E$ is bounded away from zero in an open neighborhood of zero, and the conditional distribution of $Y$ given $\bX=\bx$ admits a continuous density with respect to the Lebesgue measure.
Then there exists a sequence of random sets $\Lambda_n\subseteq\mathbb{R}^d$ such that $\mathbb{P}(\bX\in\Lambda_n)=1-o(1)$ and
\begin{align*}
    \sup_{\bx\in\Lambda_n}
    \left|
    \P\!\left(Y\le \hat{v}^{\mathrm{c}}_{\alpha,n}(\bx)| \bX=\bx\right)-\alpha
    \right|
    =o_{{p}}(1).
\end{align*}
\end{theorem}
\begin{proof}
The conclusion follows immediately by combining the $L^2$ convergence established in Theorem \ref{theorem:L2} with the additional regularity assumptions, and applying Corollary 1 of \citet{sesia2020comparison}.
\end{proof}

In addition, the results may be extended to settings without honesty by following what has been done in random forests \citep{wager2015adaptive}.

\section{Numerical Experiments}
\label{sec:numerical}

We consider a realistic portfolio risk measurement problem involving options written on multiple underlying assets, with the setting adapted from \citet{hong2017kernel}.
In particular, the asset price dynamics $\bS(t)$ are governed by a $d$-dimensional geometric Brownian motion with drifts $\mu_i'$, volatilities $\sigma_i$, and pairwise correlations $\rho_{ij}$, $i,j=1,\ldots,d$. 
Specifically, each underlying asset price evolves according to
\begin{align*}
    \frac{\mathrm{d} S_i(t)}{S_i(t)} = \mu_i' \,\mathrm{d}t + \sum_{j=1}^i A_{ij} \,\mathrm{d}B_j(t), 
    \qquad i = 1,\ldots,d,
\end{align*}
where $\{B_j(t)\}_{j=1}^d$ are independent standard Brownian motions, and the lower triangular matrix $A$ satisfies $\Sigma = A A^\top$ with the covariance matrix $\Sigma = (\Sigma_{ij})$ given by $\Sigma_{ij} = \sigma_i \sigma_j \rho_{ij}$.

By It\^{o}'s formula, the asset prices admit the explicit expression
\begin{align*}
S_i(t) = S_i(0)\exp\left\{\left(\mu_i' - \tfrac{1}{2}\sigma_i^2\right)t + \sum_{j=1}^i A_{ij} B_j(t)\right\}, \quad i=1,\ldots,d.
\end{align*}
Note that the drift parameter $\mu'$ is specified as the real-world expected return $\mu$ over the risk horizon $[0,\tau]$, and as the risk-free rate $r$ under the risk-neutral probability measure over $(\tau, T]$. 
For simplicity, we assume all assets share the same parameters and thus omit subscripts in these notations. We set $S(0) = 100$, $\mu = 0.08$, $r = 0.05$, $\sigma = 15\%$, and $\rho = 0.3$.

In our example, we set $d=4$, and the portfolio consists of twenty European call options, with five options written on each underlying asset at strike prices $K_1 = 90$, $K_2 = 95$, $K_3 = 100$, $K_4 = 105$, and $K_5 = 110$. All options are assumed to share a common maturity $T$.
We set $T = 1/12$, i.e., one month, $\tau = 1/52$, i.e., one week, and the monitoring time point for online estimation $u = 1/252$, i.e., one day.
The portfolio value at time $0$, denoted by $V(0)$, can be computed analytically using the Black--Scholes formula. The portfolio value at the risk horizon $\tau$ is given by
\begin{align*}
V(\tau) 
= \mathbb{E}\!\left[
\left.
\sum_{k=1}^d \sum_{j=1}^{5}
\mathrm{e}^{-r(T-\tau)} \big(S_k(T) - K_j\big)^+
\right|
\bX(\tau)
\right],
\end{align*}
where $\bX(\tau)=\bS(\tau)$ in this example.
Accordingly, the portfolio loss at time $\tau$ is defined as
\begin{align*}
L(\tau) = V(0) - V(\tau).
\end{align*}
Specifically, when computing $V(\tau)$, we approximate the conditional expectation using $500$ inner simulation paths, which is considered sufficient to make the inner error negligible relative to the estimation error \citep{jiang2020online,cai2020online}.

To verify the consistency of the QRF estimator in Theorems~\ref{thm:consist} and \ref{theorem:L2}, we examine the estimation errors across varying offline sample sizes. 
For the conformal QRF estimator, we use $70\%$ of the offline samples for training and the remaining $30\%$ for calibration.
Furthermore, we vary the confidence level $\alpha$ to examine the performance of the methods in different situations. 
In particular, $\alpha$ is set to be $90.0\%$, $95.0\%$, $99.0\%$, and $99.5\%$.

We use two metrics to quantify the estimation error. The first is the mean root integrated squared error (MRISE),  which is used in empirical studies of quantile regression \citep[see, e.g.,][]{bondell2010noncrossing}, defined as
\begin{align*}
\mathrm{MRISE}
= \frac{1}{m'} \sum_{i=1}^{m'} \sqrt{\frac{1}{n'} \sum_{j=1}^{n'} \left( v_\alpha(\bx_{i,j}) - \hat v_\alpha(\bx_{i,j}) \right)^2},
\end{align*}
where $n'$ denotes the number of covariate evaluation points and $m'$ denotes the number of replications.
Here, $n'$ and $m'$ are introduced to stabilize the performance evaluation by averaging over multiple covariate points and repeated model training runs, thereby accounting for uncertainty arising from both covariate variation and model learning.
Throughout all experiments, we set $ n' = 1000$ and $ m' = 40$.

The second metric is the mean pinball loss (MPL). 
Compared with MRISE, MPL captures the asymmetric situation, which imposes a larger penalty on underestimation than on overestimation. 
Similar asymmetric metrics are also used in regulatory evaluation approaches \citep{lopez1998methods}.
Thus, an estimator that errs on the conservative side, i.e., slightly overestimates the target, tends to attain a smaller MPL.
More specifically, it is defined as
\begin{align*}
\mathrm{MPL}
= \frac{1}{m'} \sum_{i=1}^{m'}  \left( \frac{1}{n'} \sum_{i=1}^{n'}
\rho_\alpha \left( v_\alpha(\bx_{i,j}) - \hat v_\alpha(\bx_{i,j}) \right) \right), \;
\rho_\alpha(a) =
\begin{cases}
\alpha a, & a > 0, \\
(\alpha - 1) a, & a \le 0 .
\end{cases}
\end{align*}

\begin{figure}[H]
\centering
\includegraphics[width=1\textwidth]{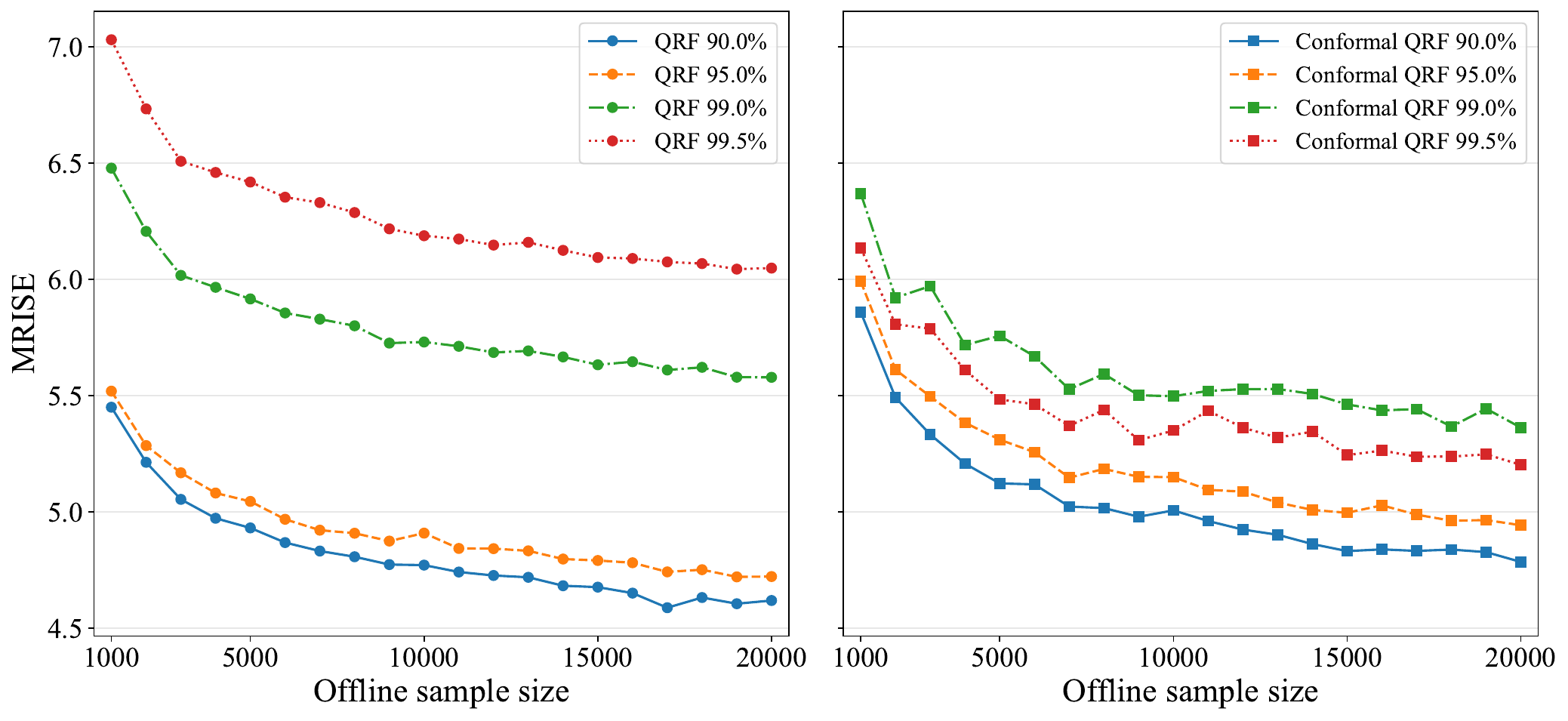}
\caption{MRISE of QRF (left) and Conformal QRF (right) under varying offline sample sizes and quantile levels.}
\label{fig:mrise_qrf_vs_conformal_qrf}
\end{figure}

Figure~\ref{fig:mrise_qrf_vs_conformal_qrf} plots the MRISE of QRF and conformal QRF across different offline sample sizes and VaR confidence levels.
For each confidence level, the MRISE of the QRF estimator decreases and converges as the offline sample size increases, consistent with the theoretical guarantees established in Theorems \ref{thm:consist} and \ref{theorem:L2}.
Moreover, it can also be observed that MRISE of QRF tends to increase as the confidence level rises, which comes from the intrinsic estimation difficulty of tail quantiles when data points in extreme regions are scarce \citep{gnecco2024extremal}.
For conformal QRF, however, the MRISE does not decrease after the calibration step, except for $\alpha=99.5\%$.
This is acceptable because the conformalization is not tailored to such a symmetric metric.
Moreover, it is reassuring to see that the conformal QRF estimator exhibits a convergence trend similar to the QRF estimator, offering empirical evidence for its reliability in practical use.

\begin{figure}[H]
\centering
\includegraphics[width=1\textwidth]{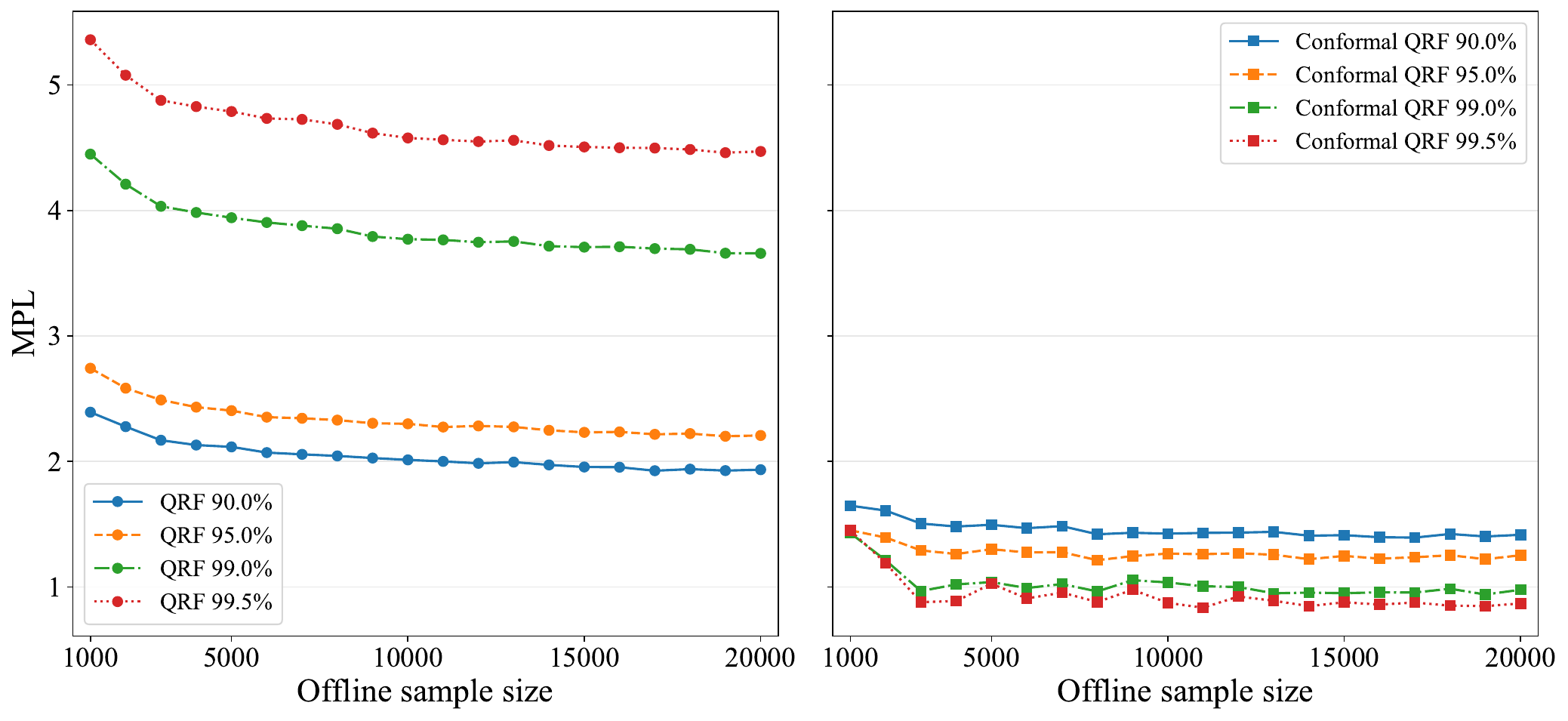}
\caption{MPL of QRF (left) and Conformal QRF (right) under varying offline sample sizes and quantile levels.}
\label{fig:pinball_qrf_vs_conformal_qrf}
\end{figure}

Figure~\ref{fig:pinball_qrf_vs_conformal_qrf} compares the MPL of QRF and conformal QRF over varying offline sample sizes and VaR confidence levels.
It can be observed that conformal QRF yields a smaller MPL than QRF, different from MRISE performance in Figure~\ref{fig:mrise_qrf_vs_conformal_qrf}, confirming the effectiveness of the designed calibration step.
Interestingly, after conformal calibration, as the confidence level becomes more extreme, the MPL decreases more markedly.
It looks a little counterintuitive because, in general, more extreme quantiles are harder to estimate and therefore incur larger errors.
In fact, this pattern comes from the definition of MPL, which penalizes overestimation by the factor $1-\alpha$ and the factor becomes negligible when $\alpha$ is large.
Hence, a more conservative estimator that tends to overestimate can achieve a smaller MPL.
Note that conformal QRF is such an estimator, so it is reasonable that it improves MPL across all confidence levels, with the improvement being most marked at the most extreme level.
In addition, MPL values are not comparable across different confidence levels since the pinball loss depends on $\alpha$, whereas comparisons of MRISE across levels are meaningful.

\begin{figure}[H]
\centering
\includegraphics[width=1\textwidth]{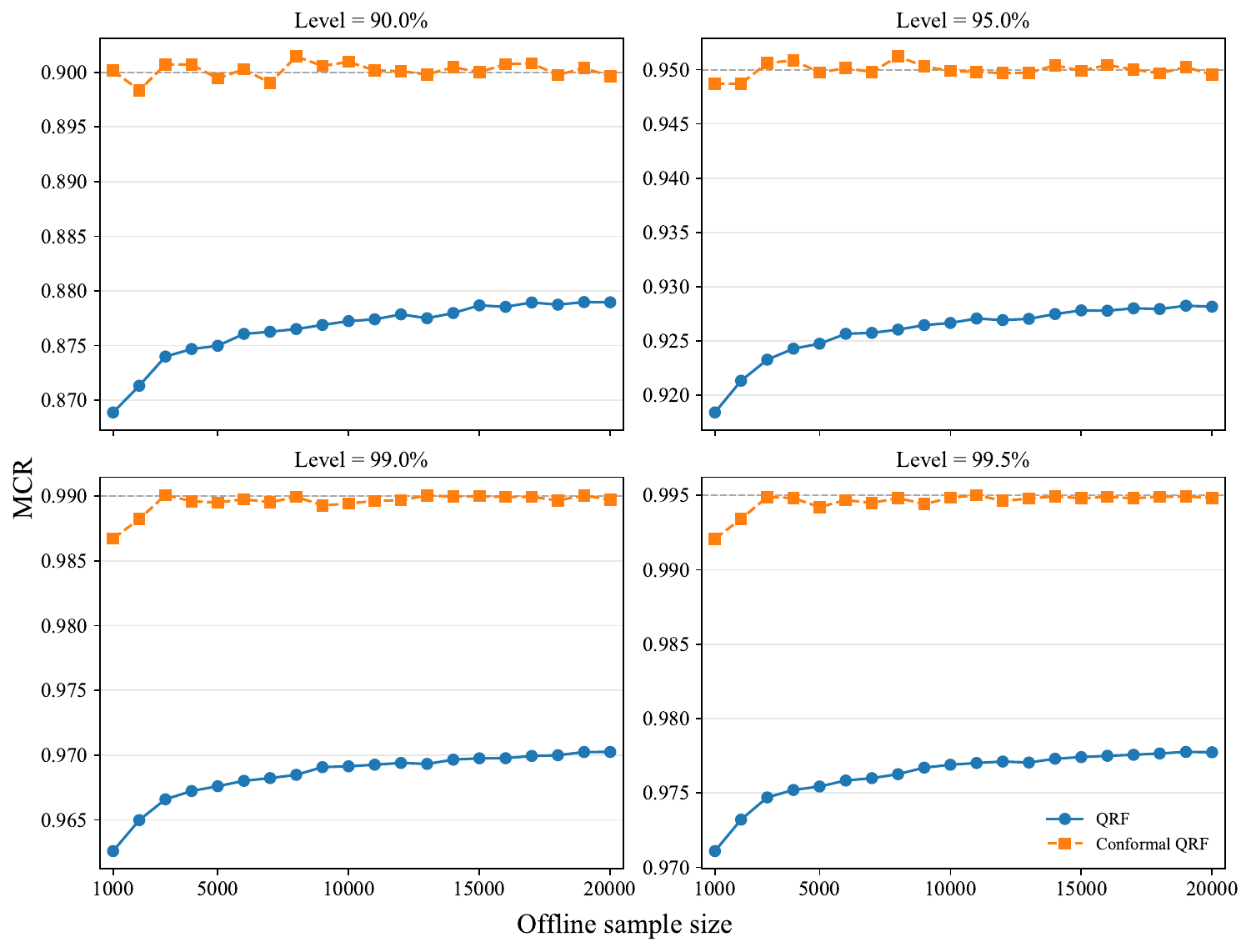}
\caption{MCR of QRF and Conformal QRF under varying offline sample sizes and quantile levels.}
\label{fig:coverage_qrf}
\end{figure}

We then investigate the coverage rate results in Theorem~\ref{thm:asymp_conditional_conformal_var}. 
Specifically, for the evaluated covariate $\bx$, we generate $M = 25{,}000$ independent realizations of the portfolio loss $\{Y^{(m)}\}_{m=1}^M$ using the same data-generating process as in the construction of the offline samples. 
Then the coverage rate (CR) is calculated by 
\begin{align*}
\mathrm{CR}
= \frac{1}{n'} \sum_{i=1}^{n'} \left(\frac{1}{M} \sum_{m=1}^{M}
\mathbf 1 \{ Y_{i}^{(m)} \le \hat v_\alpha(\bx_{i}) \}\right).
\end{align*}
For clarity, we report the mean CR (MCR), obtained by averaging CR over $m'$ covariates, since results for individual covariate points are much noisy despite being consistent with the overall trend.

The MCRs of QRF and conformal QRF with respect to different offline sample sizes for several VaR confidence levels are plotted in Figure~\ref{fig:coverage_qrf}. 
It can be seen that the conformal QRF estimator achieves the target level, which verifies the coverage guarantee in Theorem~\ref{thm:asymp_conditional_conformal_var}.
Moreover, the QRF estimator approaches the target level as the offline sample size increases but remains significantly undercoverage, which is exactly the motivation for us to introduce the conformal calibration step to recover the coverage rate and to mitigate the undesirable underestimation, and can serve as the supporting empirical evidence.

\section{Conclusion}\label{sec:conclusion}
In this paper, we have studied the real-time VaR estimation problem for financial portfolios under the OSOA framework, proposed the QRF estimator and the conformal QRF estimator, and analyzed their theoretical properties. 
We have demonstrated the efficiency and the reliability of our method through numerical results.

\section*{CRediT authorship contribution statement}
\textbf{Du-Yi Wang}: Writing – original draft, Methodology, Software, Visualization, Data Curation.
\textbf{Guo Liang}: Writing – review \& editing, Methodology, Validation.
\textbf{Kun Zhang}: Writing – review \& editing, Conceptualization, Methodology, Funding acquisition, Supervision, Project administration, Validation.
\textbf{Qianwen Zhu}: Writing – original draft, Methodology, Conceptualization, Investigation, Software, Formal analysis.

\section*{Declaration of competing interest}
None declared.

\section*{Declaration of generative AI and AI-assisted technologies in the manuscript preparation process}
During the preparation of this work, the authors used ChatGPT in order to polish the language. After using this tool, the authors reviewed and edited the content as needed and take full responsibility for the content of the published article.

\section*{Data availability statement}
Data will be made available on request.



\bibliographystyle{elsarticle-harv} 
\bibliography{sample}

@article{lopez1998methods,
  title={Methods for evaluating value-at-risk estimates},
  author={Lopez, Jose A},
  journal={Economic Policy Review},
  volume={4},
  number={3},
  pages     = {119--124},
  month     = oct,
  publisher = {Federal Reserve Bank of New York},
  year={1998}
}

@techreport{BCBS1996Backtesting,
  author       = {{Basel Committee on Banking Supervision}},
  title        = {Supervisory framework for the use of {``backtesting''} in conjunction with the internal models approach to market risk capital requirements},
  institution  = {Bank for International Settlements},
  address      = {Basel},
  year         = {1996}
}

@article{Campbell2006,
journal = {Journal of Risk},
publisher = {Risk Publications},
title = {A review of backtesting and backtesting procedures},
author={Sean D. Campbell},
pages={1--7},
volume = {9},
year = {2006},
}

@article{huang2022nonparametric,
  title={Nonparametric mean-lower partial moment model and enhanced index investment},
  author={Huang, Jinbo and Li, Yong and Yao, Haixiang},
  journal={Computers \& Operations Research},
  volume={144},
  pages={105814},
  year={2022},
  publisher={Elsevier}
}

@article{valle2025portfolio,
  title={Portfolio optimisation: Bridging the gap between theory and practice},
  author={Valle, Cristiano Arbex},
  journal={Computers \& Operations Research},
  volume={175},
  pages={106918},
  year={2025},
  publisher={Elsevier}
}

@article{bondell2010noncrossing,
  title={Noncrossing quantile regression curve estimation},
  author={Bondell, Howard D and Reich, Brian J and Wang, Huixia},
  journal={Biometrika},
  volume={97},
  number={4},
  pages={825--838},
  year={2010},
  publisher={Oxford University Press}
}

@article{athey2016recursive,
  title={Recursive partitioning for heterogeneous causal effects},
  author={Athey, Susan and Imbens, Guido},
  journal={Proceedings of the National Academy of Sciences},
  volume={113},
  number={27},
  pages={7353--7360},
  year={2016},
  publisher={National Academy of Sciences}
}

@book{wainwright2019high,
  title={High-dimensional Statistics: A Non-asymptotic Viewpoint},
  author={Wainwright, Martin J},
  volume={48},
  year={2019},
  publisher={Cambridge University Press}
}

@book{Durrett_2019, 
edition={5th}, 
title={Probability: Theory and Examples}, 
publisher={Cambridge University Press}, 
author={Durrett, Rick}, 
year={2019}, 
}

@article{wager2015adaptive,
  title={Adaptive concentration of regression trees, with application to random forests},
  author={Wager, Stefan and Walther, Guenther},
  journal={arXiv preprint arXiv:1503.06388},
  year={2015}
}

@inproceedings{romano2019conformalized,
  title={Conformalized quantile regression},
  author={Romano, Yaniv and Patterson, Evan and Candes, Emmanuel},
  booktitle={Advances in Neural Information Processing Systems},
  volume={32},
  pages={3543--3553},
  year={2019}
}

@article{sesia2020comparison,
  title={A comparison of some conformal quantile regression methods},
  author={Sesia, Matteo and Cand{\`e}s, Emmanuel J},
  journal={Stat},
  volume={9},
  number={1},
  pages={e261},
  year={2020},
  publisher={Wiley Online Library}
}

@article{zhang2025conditional,
  title={Conditional Generative Modeling for Enhanced Credit Risk Management in Supply Chain Finance},
  author={Zhang, Qingkai and Hong, L Jeff and Yan, Houmin},
  journal={Naval Research Logistics, Publish Online},
  year={2025},
  publisher={Wiley Online Library}
}

@article{CevidMichelNafBuhlmannMeinshausen2022,
  title   = {Distributional Random Forests: Heterogeneity Adjustment and Multivariate Distributional Regression},
  author  = {Cevid, Domagoj and Michel, Loris and N{\"a}f, Jeffrey and B{\"u}hlmann, Peter and Meinshausen, Nicolai},
  journal = {Journal of Machine Learning Research},
  volume  = {23},
  number  = {333},
  pages   = {1--79},
  year    = {2022}
}

@article{wager2018estimation,
  title={Estimation and inference of heterogeneous treatment effects using random forests},
  author={Wager, Stefan and Athey, Susan},
  journal={Journal of the American Statistical Association},
  volume={113},
  number={523},
  pages={1228--1242},
  year={2018},
  publisher={Taylor \& Francis}
}

@inproceedings{grinsztajn2022tree,
  title={Why do tree-based models still outperform deep learning on typical tabular data?},
  author={Grinsztajn, L{\'e}o and Oyallon, Edouard and Varoquaux, Ga{\"e}l},
  booktitle={Advances in Neural Information Processing Systems},
  volume={35},
  pages={507--520},
  year={2022}
}

@article{taylor2000quantile,
  title={A quantile regression neural network approach to estimating the conditional density of multiperiod returns},
  author={Taylor, James W},
  journal={Journal of Forecasting},
  volume={19},
  number={4},
  pages={299--311},
  year={2000},
  publisher={Wiley Online Library}
}

@book{jorion2006,
  title={Value at Risk: The New Benchmark for Managing Financial Risk},
  author={Philippe Jorion},
  year={2006},
  publisher = {McGraw-Hill},
  address   = {New York, NY},
  edition   = {3rd}
}

@article{koenker1994quantile,
  title={Quantile smoothing splines},
  author={Koenker, Roger and Ng, Pin and Portnoy, Stephen},
  journal={Biometrika},
  volume={81},
  number={4},
  pages={673--680},
  year={1994},
  publisher={Oxford University Press}
}

@article{koenker1978regression,
  title={Regression quantiles},
  author={Koenker, Roger and Bassett Jr, Gilbert},
  journal={Econometrica},
  volume  = {46},
  number  = {1},
  pages={33--50},
  year={1978},
  publisher={JSTOR}
}

@article{yu1998local,
  title={Local linear quantile regression},
  author={Yu, Keming and Jones, MC1614628},
  journal={Journal of the American Statistical Association},
  volume={93},
  number={441},
  pages={228--237},
  year={1998},
  publisher={Taylor \& Francis}
}

@article{friedman2001greedy,
  title={Greedy function approximation: A gradient boosting machine},
  author={Friedman, Jerome H},
  journal={The Annals of Statistics},
  volume  = {29},
  number  = {5},
  pages={1189--1232},
  year={2001},
  publisher={JSTOR}
}

@article{broadie2015regression,
  title={Risk Estimation via Regression},
  author={Broadie, Mark and Du, Yiping and Moallemi, Ciamac C},
  journal={Operations Research},
  volume={63},
  number={5},
  pages={1077--1097},
  year={2015},
  publisher={INFORMS}
}

@article{gordy2010nested,
  title={Nested Simulation in Portfolio Risk Measurement},
  author={Gordy, Michael B and Juneja, Sandeep},
  journal={Management Science},
  volume={56},
  number={10},
  pages={1833--1848},
  year={2010},
  publisher={INFORMS}
}

@article{zhang2022bootstrap,
  title={Bootstrap-Based Budget Allocation for Nested Simulation},
  author={Zhang, Kun and Liu, Guangwu and Wang, Shiyu},
  journal={Operations Research},
  volume={70},
  number={2},
  pages={1128--1142},
  year={2022},
  publisher={INFORMS}
}

@article{hong2017kernel,
  title={Kernel Smoothing for Nested Estimation with Application to Portfolio Risk Measurement},
  author={Hong, L Jeff and Juneja, Sandeep and Liu, Guangwu},
  journal={Operations Research},
  volume={65},
  number={3},
  pages={657--673},
  year={2017},
  publisher={INFORMS}
}

@article{lai2024simulating,
  title={Simulating Confidence Intervals for Conditional Value-at-Risk via Least-Squares Metamodels},
  author={Lai, Qidong and Liu, Guangwu and Zhang, Bingfeng and Zhang, Kun},
  journal={INFORMS Journal on Computing},
  volume={37},
  number={4},
  pages = {1087-1105},
  year={2024},
  publisher={INFORMS}
}

@article{wang2024smooth,
author = {Wang, Wenjia and Wang, Yanyuan and Zhang, Xiaowei},
title = {Smooth Nested Simulation: Bridging Cubic and Square Root Convergence Rates in High Dimensions},
journal = {Management Science},
volume = {70},
number = {12},
pages = {9031-9057},
year = {2024}
}

@inproceedings{cai2020online,
  title={Online risk measure estimation via natural gradient boosting},
  author={Cai, Xiaoting and Yang, Yang and Jiang, Guangxin},
  booktitle={2020 Winter Simulation Conference (WSC)},
  pages={2341--2352},
  year={2020},
  organization={IEEE}
}

@article{lin2025contextual,
  title={Contextual Strongly Convex Simulation Optimization: Optimize then Predict with Inexact Solutions},
  author={Lin, Nifei and Luo, Heng and Hong, L Jeff},
  journal={arXiv preprint arXiv:2512.06270},
  year={2025}
}

@article{gnecco2024extremal,
  title={Extremal random forests},
  author={Gnecco, Nicola and Terefe, Edossa Merga and Engelke, Sebastian},
  journal={Journal of the American Statistical Association},
  volume={119},
  number={548},
  pages={3059--3072},
  year={2024},
  publisher={Taylor \& Francis}
}

@article{GRF2019,
author = {Susan Athey and Julie Tibshirani and Stefan Wager},
title = {{Generalized random forests}},
volume = {47},
journal = {The Annals of Statistics},
number = {2},
publisher = {Institute of Mathematical Statistics},
pages = {1148--1178},
year = {2019}
}

@inproceedings{duan2020ngboost,
  title={Ngboost: Natural gradient boosting for probabilistic prediction},
  author={Duan, Tony and Anand, Avati and Ding, Daisy Yi and Thai, Khanh K and Basu, Sanjay and Ng, Andrew and Schuler, Alejandro},
  booktitle={International Conference on Machine Learning},
  pages={2690--2700},
  year={2020},
  organization={PMLR}
}

@article{velthoen2023gradient,
  title={Gradient boosting for extreme quantile regression},
  author={Velthoen, Jasper and Dombry, Cl{\'e}ment and Cai, Juan-Juan and Engelke, Sebastian},
  journal={Extremes},
  volume={26},
  number={4},
  pages={639--667},
  year={2023},
  publisher={Springer}
}

@article{koenker2017quantile,
  title={Quantile regression: 40 years on},
  author={Koenker, Roger},
  journal={Annual Review of Economics},
  volume={9},
  number={1},
  pages={155--176},
  year={2017},
  publisher={Annual Reviews}
}

@article{meinshausen2006quantile,
  title={Quantile regression forests.},
  author={Meinshausen, Nicolai},
  journal={Journal of Machine Learning Research},
  volume={7},
  number={6},
  pages   = {983--999},
  year={2006}
}

@article{angelopoulos2023conformal,
  title={Conformal prediction: A gentle introduction},
  author={Angelopoulos, Anastasios N and Bates, Stephen},
  journal={Foundations and Trends{\textregistered} in Machine Learning},
  volume={16},
  number={4},
  pages={494--591},
  year={2023},
  publisher={Now Publishers, Inc.}
}

@article{keslin2025ranking,
  title={Ranking and contextual selection},
  author={Keslin, Gregory and Nelson, Barry L and Pagnoncelli, Bernardo and Plumlee, Matthew and Rahimian, Hamed},
  journal={Operations Research},
  volume={73},
  number={5},
  pages={2695--2707},
  year={2025},
  publisher={INFORMS}
}

@article{jiang2024real,
  title={Real-Time Derivative Pricing and Hedging with Consistent Metamodels},
  author={Jiang, Guangxin and Hong, L Jeff and Shen, Haihui},
  journal={INFORMS Journal on Computing},
  volume={36},
  number={5},
  pages={1168--1189},
  year={2024},
  publisher={INFORMS}
}

@article{jiang2020online,
  title={Online risk monitoring using offline simulation},
  author={Jiang, Guangxin and Hong, L Jeff and Nelson, Barry L},
  journal={INFORMS Journal on Computing},
  volume={32},
  number={2},
  pages={356--375},
  year={2020},
  publisher={INFORMS}
}

@article{hong2019offline,
  title={Offline simulation online application: A new framework of simulation-based decision making},
  author={Hong, L Jeff and Jiang, Guangxin},
  journal={Asia-Pacific Journal of Operational Research},
  volume={36},
  number={06},
  pages={1940015},
  year={2019},
  publisher={World Scientific}
}

@inproceedings{luo2024reliable,
  title={Reliable Online Decision Making with Covariates},
  author={Luo, Heng and Liang, Zhiyang and Hong, L Jeff},
  booktitle={2024 Winter Simulation Conference (WSC)},
  pages={3265--3276},
  year={2024},
  organization={IEEE}
}

@inproceedings{Mohri2024,
  title = 	 {Language models with conformal factuality guarantees},
  author =       {Mohri, Christopher and Hashimoto, Tatsunori},
  booktitle = 	 {International Conference on Machine Learning},
  pages = 	 {36029--36047},
  year = 	 {2024},
  volume = 	 {235},
  series = 	 {Proceedings of Machine Learning Research},
  month = 	 {21--27 Jul},
  publisher =    {PMLR},
}

@inproceedings{Cherian2024,
  title={Language models with conformal factuality guarantees},
  author={Cherian, John and Gibbs,  Isaac and Candes, Emmanuel},
  booktitle={Advances in Neural Information Processing Systems},
  year={2024},
  pages={114812-114842}
}






\end{document}